\documentclass[letterpaper, 10 pt, conference]{ieeeconf}  

\IEEEoverridecommandlockouts                              
\usepackage{algorithm}
\usepackage[noend]{algorithmic}
\usepackage{graphicx}
\usepackage{siunitx}
\usepackage[utf8]{inputenc}
\usepackage[T1]{fontenc}
\usepackage{amsfonts}

\usepackage{amsthm}
\newcommand{\fig}[1]{Fig.~\ref{#1}}

\newcommand{\eq}[1]{(\ref{#1})}

\newcommand{\alg}[1]{Alg.~\ref{#1}}
\newtheorem{thm}{Theorem}[section]

\newtheorem{cor}{Corollary}

\usepackage{epsfig} 
\usepackage{amsmath} 
\usepackage{amssymb}  
\usepackage{subcaption}
\usepackage{caption}
\usepackage{color}
\usepackage{theoremref}

\DeclareMathOperator*{\argmin}{argmin}   
\title{\LARGE\bf LTO: Lazy Trajectory Optimization with Graph-Search Planning for High DOF Robots in Cluttered Environments}

\author{Yuki Shirai$^{1, 2}$, Xuan Lin$^{1}$, Ankur Mehta$^{2}$, and Dennis Hong$^{1}$
\thanks{The work of Y. Shirai was partially supported by the Funai Foundation for Information Technology and the Ezoe Memorial Recruit Foundation. 
}
\thanks{$^{1}$Y. Shirai, X. Lin, and D. Hong are with the Department of Mechanical and Aerospace Engineering, University of California, Los Angeles, CA 90095, USA.
        {\tt\small yukishirai4869@g.ucla.edu}, {\tt\small maynight@ucla.edu}, {\tt\small dennishong@ucla.edu}}
\thanks{$^{2}$Y. Shirai and A. Mehta are with the Department of Electrical and Computer Engineering, University of California, Los Angeles, CA 90095, USA.
        {\tt\small yukishirai4869@g.ucla.edu}, {\tt\small mehtank@ucla.edu}
        }%
}

\begin{document}
\maketitle
\thispagestyle{empty}
\pagestyle{empty}


\begin{abstract}
Although Trajectory Optimization (TO) is one of the most powerful motion planning tools, it suffers from expensive computational complexity as a time horizon increases in cluttered environments. It can also fail to converge to a globally optimal solution.
In this paper, we present Lazy Trajectory Optimization (LTO) that unifies local short-horizon TO and global Graph-Search Planning (GSP) to generate a long-horizon global optimal trajectory. 
LTO solves TO with the same constraints as the original long-horizon TO with improved time complexity. 
We also propose a TO-aware cost function that can balance both solution cost and planning time. 
Since LTO solves many nearly identical TO in a  roadmap, it can provide an informed warm-start for TO to accelerate the planning process. We also present proofs of the computational complexity and optimality of LTO. Finally, we demonstrate LTO's performance on motion planning problems for a 2 DOF free-flying robot and a 21 DOF legged robot, showing that LTO outperforms existing algorithms in terms of its runtime and reliability.
\end{abstract}
%
%

\section{Introduction}
Trajectory Optimization (TO), such as the ones based on Mixed-Integer Convex Programming (MICP), solves a motion planning problem to generate an optimal trajectory while satisfying constraints.
One of the advantages TO has compared with other planners, such as Sampling-Based Planning (SBP) (e.g., 
Rapidly-exploring Random Tree) and Graph-Search Planning (GSP) (e.g., A*), is to easily formulate a wide variety of constraints, including equality constraints. 
Conversely, GSP and SBP take considerable time in a narrow passage because it is difficult to place a sufficient number of grids or samples to represent the states present \cite{sample_constraints}. 


\begin{figure}
    \centering
    \includegraphics[width=0.4599\textwidth, clip]{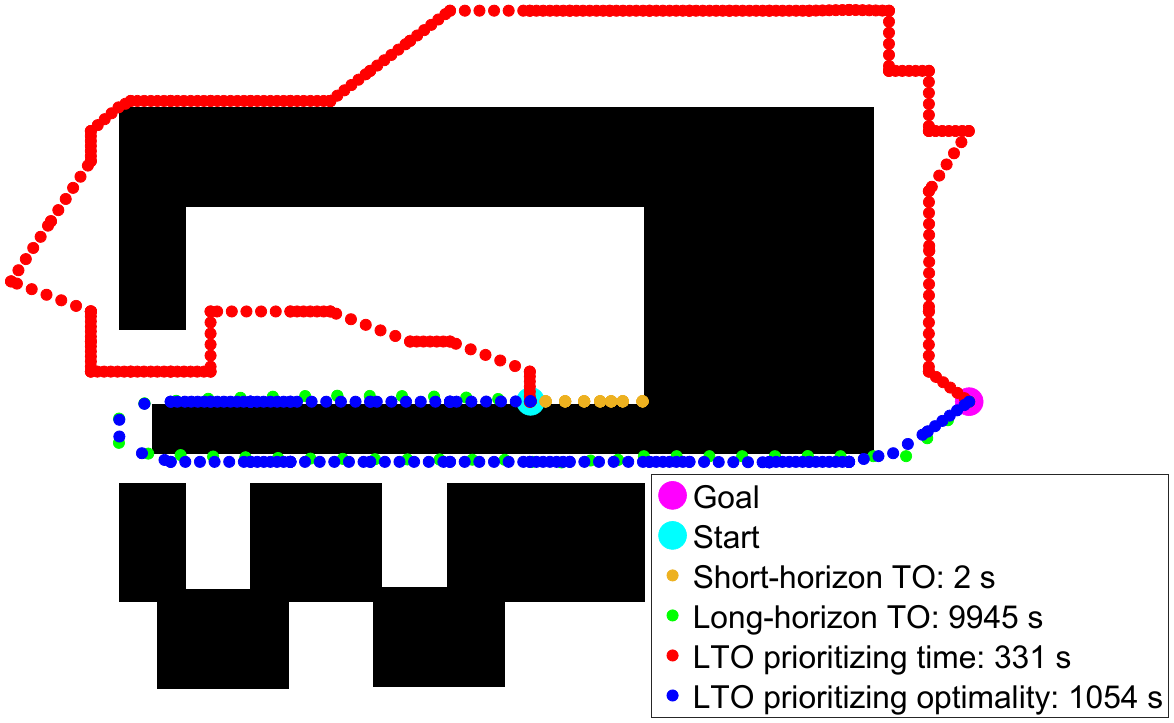}
    \caption{Generated trajectories in $\mathbb{R}^{2}$. The short-horizon TO gets stuck at the local optimum and cannot find the trajectory from start to goal. The long-horizon TO finds the optimal trajectory but takes an extended amount of time. LTO prioritizing planning time avoids the area where the time for solving TO is long due to many integer variables (i.e., obstacles). LTO prioritizing optimality of the trajectory finds the optimal solution with less planning time compared with the long-horizon TO.}
    \label{concept}
\end{figure}

However, TO has two main drawbacks: expensive computational complexity with a long horizon and convergence to local optima \cite{Risk}, \cite{IntroTO}. Long-horizon TO is indispensable for generating feasible global trajectories, but the computation time grows exponentially as the number of horizons increases. 
Model Predictive Control (MPC)
can spend less planning time than the long-horizon TO \cite{MIT_MPC}-\cite{fast_MPC}.
However, since it solves relatively short-horizon TO, 
it has a greater probability of getting stuck at local optima. 

To this end, we address Lazy Trajectory Optimization (LTO) unifying the local short-horizon TO and the global long-horizon GSP. 
LTO reasons the same constraints as the original large-horizon but with the improved time complexity.
We also propose a cost function that considers the computation time of TO to balance the optimality of the trajectory and the planning time. 
Next, based on Lazy Weighted A* (LWA*) \cite{LWA*}, we improve LWA* by making the vertex generation "lazy". In this work, "lazy" means that LTO runs TO only when it intends to evaluate the configuration and trajectories. 
Because LTO solves many similar TOs, 
it employs a warm-start to solve TO,
resulting in less planning time. By employing MICP as a short-horizon TO, we can analyze the computational complexity of LTO in addition to the bounded solution cost. 
%
Our contributions can be summarized as follows:
\begin{enumerate}
\item We propose LTO, a framework incorporating  GSP as a high-layer planner and TO as a low-layer planner, that
efficiently generates long-horizon global trajectories.
\item We present the cost function balancing the planning time of TO and the optimality of the trajectory.
\item We present the theoretic properties of LTO.
\item We demonstrate LTO's efficiency on motion planning problems of a 2 DOF free-flying robot and a 21 DOF legged robot. 
\end{enumerate}

\section{Related Work}

TO finds an optimal trajectory that satisfies constraints  \cite{Risk}-\cite{FASTER}. 
%
We focus on using MICP as a short-horizon TO for the following reasons. 
First, MICP can deal with nonlinear constraints with approximations and constraints involving discrete variables \cite{MIT_envelope}.
%
Another reason is that the solving time for MICP based on Branch and Bound (B\&B) \cite{branch}-\cite{LVIS} is bounded theoretically. 


Several GSPs have been studied to decrease the planning time \cite{LWA*}-\cite{SD}.
LWA* \cite{LWA*} evaluates edges only when the planner uses them. At the start of planning, LWA* does know the true edge cost and gives an optimistic value. During the planning process, when the state is selected for expansion, LWA* evaluates the true edge cost, resulting in the decreased planning time. The limitations of LWA* are that it does not consider the difficulty of edge evaluations using TO and it assumes that the true vertex configuration is known.

The hierarchical planners \cite{hierarchical1}-\cite{hierarchical2} have shown remarkable success. Our LTO is similar to these algorithms for decreasing the planning time considering dynamics. LTO could spend less planning time in cluttered environments since it simultaneously solves GSP and TO to avoid dead-ends while the hierarchical planners need to run expensive global planners again until they find feasible paths.

Several works have been proposed to have good warm-starts \cite{UAV_hybrid}-\cite{wheel_hybrid}. Compared with them, LTO utilizes past trajectories in graph to enhance the quality of wart-starts.

\section{Problem Formulation}\label{problem_formulation}


\subsection{Notation}

\begin{figure}
    \centering
    \includegraphics[width=0.29\textwidth]{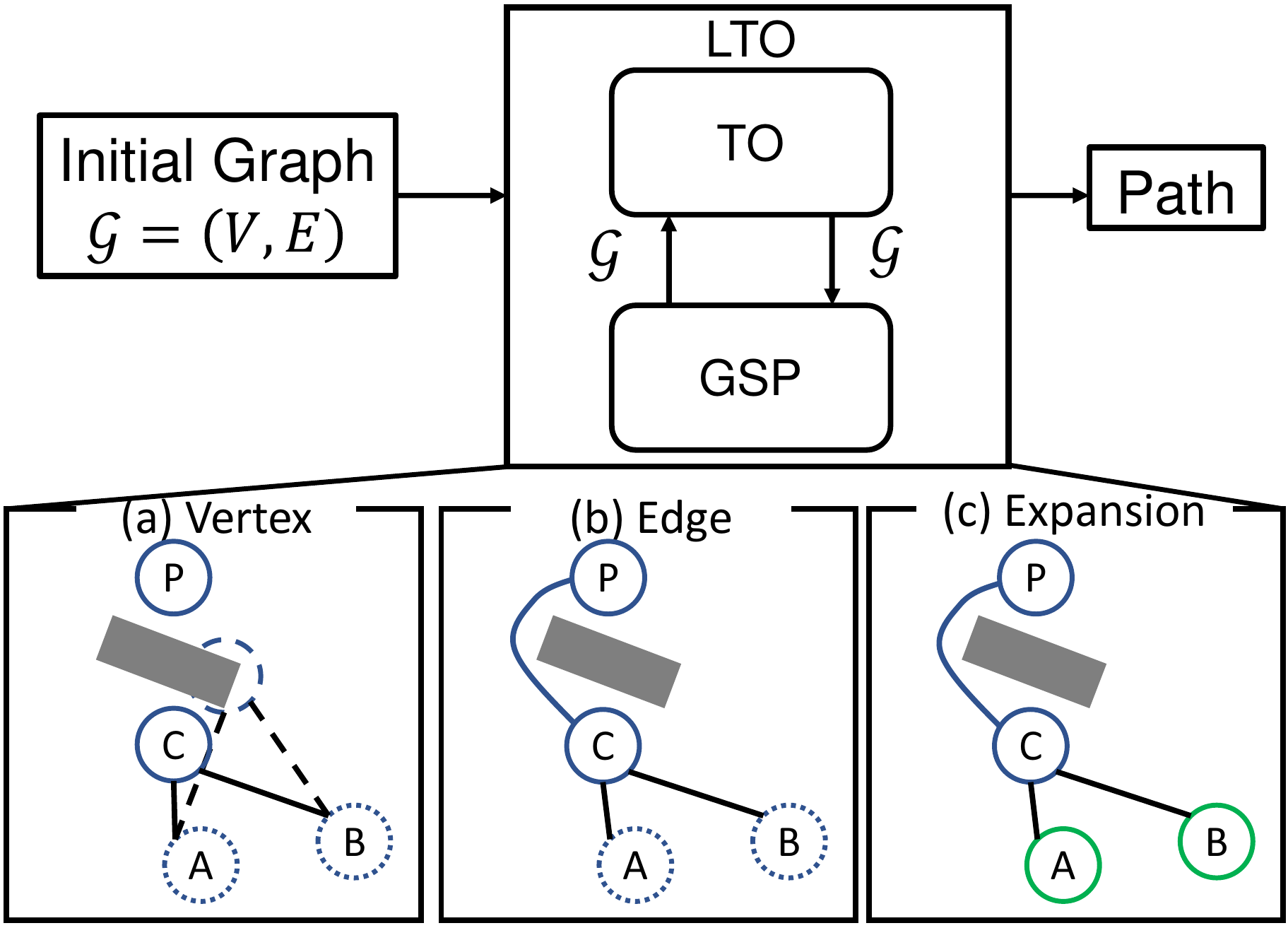}
    \caption{Overview of LTO. Given an initial graph where the true configuration of the robot and the trajectory are unknown, LTO iteratively solves TO locally to get the true vertex (configuration) and edges (trajectories) in $\mathcal{G}$. LTO performs either expansion of the vertex, gets the true vertex using TO, or gets the true edge using TO. Assume the vertex P is the current vertex LTO chooses from an open list. (a): LTO gets the true configuration of the vertex C. (b): it generates the true trajectory from the vertex P to C. (c): it expands the vertex C and inserts the vertices A and B into the open list.}
    \label{overview_fig}
\end{figure}

LTO solve TO with GSP as shown in \fig{overview_fig}. Let $\mathcal{G}=(V, E)$ be a priori unknown graph consisting of vertices $V = \left(v_1, v_2,  \ldots\right)$ and edges $E = (e(v_i \leftrightarrow v_j) \forall {i}, \forall {j})$. 
Each vertex represents the state of a robot and each edge represents the trajectory of the robot between the vertices. We make voxels in the continuous domain, such that each vertex is in each voxel shown in \fig{edge}. 
At the start of planning, we connect every two vertices if their $\ell_{\infty}$ norm is less than or equal to $r$. Let $K$ be the number of intervals along each axis. Let $i$ be the number of voxels along each axis to produce a hypercube region where LTO solves the edge (see \fig{edge}). 
%




\subsection{Graph Structure}\label{graph_structure}

\begin{figure}
    \centering
    \includegraphics[width=0.32\textwidth]{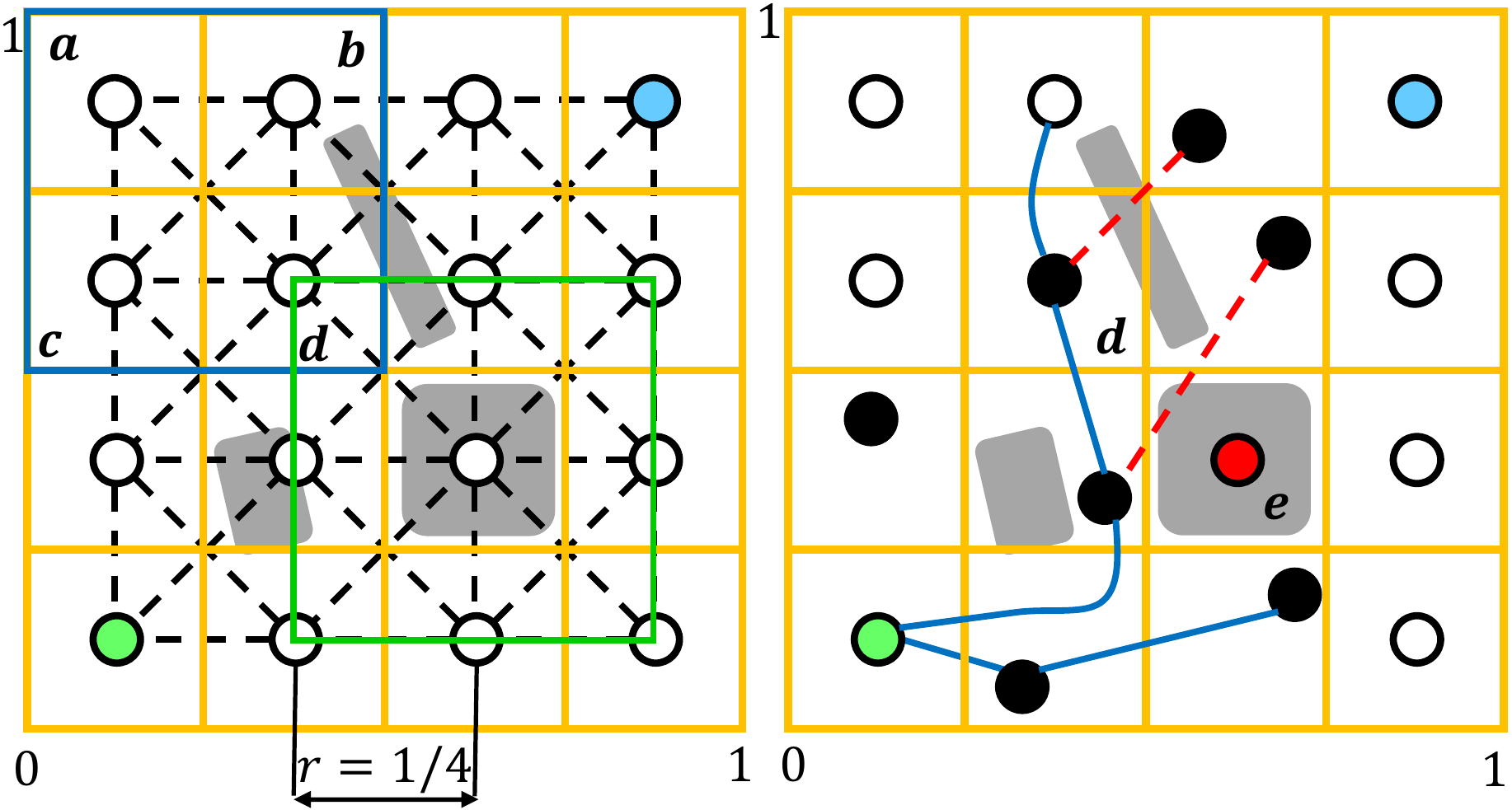}
    \caption{The planning procedure by LTO in $[0, 1]^2$ where $K=4, i=1, r = 1/4$. The left figure shows a graph at the start of planning where each vertex is inside the orange voxel and each edge is represented as a black dashed line. When LTO investigates the vertex, it solves TO within this voxel. When it investigates the edge, it solves TO within the associated voxels. For example, when it investigates the edge from the vertex (a) to (d), it solves TO in the voxels (a), (b), (c), (d), considering the constraints. The right figure shows the current graph after iterations. For simplicity, it does not show the black dashed lines. The true vertex and the true edges found by TO are shown as the black circles and blue lines. The infeasible vertex and edge judged by TO are shown as the red circle and red lines. For example, the vertex (e) is removed since TO cannot find the feasible configuration in the voxel (e).}
    \label{edge}
\end{figure}


We solve TO within associated voxels and only use the constraints within voxels as shown in \fig{edge}. It means
we remove several domain-specific constraints, such as obstacle-avoidance constraints, if they are outside the voxels, resulting in the decreased solving time. Because we keep each vertex and edge within the voxels, the constraints outside the voxels do not influence the generated vertex and edge in voxels. 
%


We describe how we build our graph. In the beginning, 
we make voxels in the continuous domain and place a vertex in each voxel. 
If we place a vertex to represent the robot's state without considering the feasibility of the state, the probability of the state associated with the vertex being infeasible would be high. If we use TO to place a vertex, TO considers constraints so that LTO can place the vertex in a feasible region. Thus, it uses TO to get the true configuration associated with the vertex and update $\mathcal{G}$. To execute TO, it associates each grid voxel with a continuous state of the robot. 
%
For edge generation,
it uses TO to get the true trajectory associated with the edge and update $\mathcal{G}$. 
LTO solves TO in the hypercube, consisting of corner points of the voxels where the target vertices are located.


\subsection{Mixed-Integer Convex Programs}\label{MICP_Formulation}
The MICP to generate edges in $\mathcal{G}$ is given by:
\begin{equation}
   \begin{array}{cl}
\text { minimize } & c_T(x_N, z)+\sum_{t=0}^{N-1}c_i(x_t, z)\\
\text { subject to } & f_i(x_t, z) \leq 0, \quad t=0, \ldots, N-1 \\
& x_{\min } \leq x_{t} \leq x_{\max }, \quad t=0, \ldots, N \\
& x_{0}=x_{s}, \quad x_{N}=x_{g}  \\
& x_{t} \in \mathcal{X}_{\text { }}, \quad t=0, \ldots, N\\
&z \in\{0,1\}^{n_{z}}
\label{MICP_standard}
\end{array}
\end{equation}
where $x_t$ are the states of the robot at time $t$, $z$ are binary decision variables, $c_T, c_i, f_i$ are convex functions, and $\mathcal{X}$ is the convex set. 
When finding the edge in $\mathcal{G}$, we solve \eq{MICP_standard}, where $x_{s}, x_{g}$ are the state of the start vertex and the state of the goal vertex, respectively. 
%
When finding the vertex in $\mathcal{G}$, we solve \eq{MICP_standard} with $N = 0$ without $\sum_{t=0}^{N-1} c_{i}\left(x_{t}, z\right)$. 

\subsection{Warm-Start Strategy}\label{warm-start-strategy}
Let $v_p$ and $v_c$ be the start and goal state in the trajectory LTO tries to generate in $\mathcal{G}$. LTO searches the most similar trajectory in $\mathcal{G}$ based on the deviation cost as follows:
\begin{equation}
    d_{\operatorname{cost}} = \left\|v_{p}-v_{i}\right\| + \left\|v_{c}-v_{j}\right\|
\label{warm_start_eq}
\end{equation}
where $v_i$ and $v_j$ are the start and goal of other trajectories in $\mathcal{G}$. We assume that a trajectory in $\mathcal{G}$ with a close start and goal designs a similar trajectory, enabling the robot to be aware of constraints.
Hence, when LTO tries to generate an edge not investigated by TO yet, it uses the edges already generated by TO as initial guesses if $d_{\operatorname{cost}}$  is lower than the threshold. 
%
As time passes, it solves more edges and this information enhances the quality of the warm-start for the current trajectory generation.

\section{Lazy Trajectory Optimization}\label{LazyARA*}
We present LTO that unifies TO and GSP.
We employ LWA* as GSP of LTO 
and improve LWA* by proposing a new cost function that considers the difficulty of TO with the guaranteed suboptimality bound. We also delay vertex validation using TO until the planner intends to expand the vertex since executing TO for all the voxels to validate the vertices is demanding. 
%
%
%
%
The high-level process of LTO is shown in \fig{overview_fig}.
Given a uniform grid graph where LTO does not know the true vertex and edge, LTO performs either an expansion, a vertex validation, or an edge generation. 

 We use the notation  $X \stackrel{+}{\leftarrow}\{\mathbf{x}\}$ and $X \stackrel{-}{\leftarrow}\{\mathbf{x}\}$ to show the compounding operations $X \leftarrow X \cup \mathbf{x}$ and $X \leftarrow X \setminus \mathbf{x}$, respectively. $Q_o, Q_c$ are priority queues to maintain the states discovered but not expanded and the expanded states.
%
 $\hat g(v)$, $\hat h(v)$, and $\hat f(v)$ are estimates of cost-to-come, cost-to-go, and cost from the start to goal through $v$, respectively. We use $\hat h(v)$ as:  $\hat h(v)=\left\|v_{i}-v_{\text{goal}}\right\|_2$.   $\operatorname{TrueVertex}$ and $\operatorname{TrueEdge}$ show if a state $v$ has the true configuration and the true edge cost. 
$\operatorname{Conf}(v)$ represents the true configuration of the vertex.

\subsection{TO-Aware Cost}
We propose the TO-aware cost as follows:
\begin{equation}
\begin{array}{l}
c\left(v_{1}, v_{2}\right)=\left\|v_{1}-v_{2}\right\|_2, 
c_{TO}\left(v_{1}, v_{2}\right)= (1+\omega ^{n_i}) c\left(v_{1}, v_{2}\right) 
\end{array}\label{cost_lazy}
\end{equation}
where $c$ is the cost of the edge using the Euclidean distance and $c_{TO}$ is the inflated cost considering the time complexity of TO. $n_i$ is the number of discrete decision variables associated with edge generation and we count $n_i$ within the associated voxels. 
$\omega$ is a user-defined inflation factor. 

%

%
$c_{TO}$ can be very large so that LTO may not enthusiastically investigate the edge in the voxels with many discrete variables. 
In other words, $\omega$ is a tuning knob that balances the optimality of a trajectory and the planning time.
We use $c_{TO}$ as the cost of the edge if it is not investigated by TO and use $c$ if it is investigated by TO. 


While many works try to minimize the number of edge evaluations \cite{LazySP}, only a few papers discuss the "difficulty" of the edge evaluation. Recognizing that the running time of the edge generation by TO grows exponentially as the number of discrete variables increases \cite{MIT_envelope}, \cite{Ponton}, we propose $c_{TO}$.



\subsection{Main Loop (\alg{alg2})}

  \begin{algorithm}[t]
  \small 
  \algsetup{linenosize=\small}
 \caption{LTO($\mathcal{G}$, $v_{\operatorname{start}}$, $v_{\operatorname{goal}}$)}
 \label{alg2}
 \begin{algorithmic}[1]
 \STATE $Q_o \leftarrow v_{\operatorname{start}}$, $Q_c \leftarrow \emptyset$, $\hat{g}\left(v_{\operatorname{start}}\right)=0$, $\hat{g}(v) \leftarrow \infty$ \label{line0}
 \STATE $\operatorname{TrueVertex}(v)\leftarrow \text{False}, \operatorname{TrueEdge}(v)\leftarrow \text{False}$
 \STATE $\operatorname{TrueVertex}(v_{\operatorname{start,goal}})\leftarrow \text{True}, \operatorname{TrueEdge}(v_{\operatorname{start}})\leftarrow \text{True}$
  \WHILE{$\hat f\left(v_{\text {goal}}\right)>\min _{v \in Q_o}(\hat f(v))$}\label{line1}
  \STATE $v = \argmin_{v \in Q_o}(\hat f(v)), Q_o \stackrel{-}{\leftarrow}\{v\} $\label{extract}
  \IF{$v == v_{\operatorname{goal}}$ }
  \RETURN ReconstructPath($v_{\operatorname{start}}$, $v_{\operatorname{goal}}$)
  \ENDIF \label{line01}
  \IF{$v \in Q_c$}\label{line3}
  \STATE CONTINUE\label{line4}
\ELSIF{$\operatorname{TrueVertex}(v)$}\label{line5}
  \IF{$\operatorname{TrueEdge}(v)$}\label{line6}
  \STATE $Q_o, Q_c=\operatorname{Expansion}(\mathcal{G}, Q_o, Q_c, v)$
  \ELSE\label{line22}
  \STATE $\mathcal{G}, Q_o, Q_c=\operatorname{UpdateEdge}(\mathcal{G}, Q_o, Q_c, v)$
  \ENDIF\label{line36}
  \ELSE\label{line37}
  \STATE $\mathcal{G}, Q_o, Q_c=\operatorname{UpdateVertex}(\mathcal{G}, Q_o, Q_c, v)$
  \ENDIF\label{line46}
  \ENDWHILE\label{line47}
  \RETURN No Path Exists
 \end{algorithmic} 
 \end{algorithm}

Lines~\ref{line0}-\ref{line01} are typical of A*. We iteratively remove the cheapest state in $Q_o$ until the goal is chosen. Lines~\ref{line3}-\ref{line4} are from LWA*, showing that a state is not expanded again if it is already expanded and continues to the next iteration of the while loop. Lines~\ref{line5}-\ref{line46} are new. $\operatorname{TrueVertex}(v)$ and $\operatorname{TrueEdge}(v)$ check if the expanded state $v$ has the true configuration and the true edge cost, respectively. 

\subsection{Expansion (\alg{expansion})}

  \begin{algorithm}[t]
    \small 
  \algsetup{linenosize=\small}
 \caption{$\operatorname{Expansion}(Q_o, Q_c, v)$}
 \label{expansion}
 \begin{algorithmic}[1]
  \STATE $Q_c \stackrel{+}{\leftarrow}\{v\}$, $S=\operatorname{GetS u c c e s s o r s}(v)$\label{line7}
  \FORALL{$v^{\prime} \in S$}\label{line9}
  \STATE parent$\left(v^{\prime}\right)=v$\label{line10}
  \IF{$\exists v^{\prime \prime} \in (Q_o \OR Q_c) \text{s.t.}$ TrueVertex($v^{\prime \prime}$) $\AND \text{Conf}(v^{\prime}) = \text{Conf}(v^{\prime \prime})$}\label{200}
  \STATE $\hat g\left(v^{\prime}\right)=\hat g\left(\text {parent}\left(v^{\prime}\right)\right)+c_{TO}\left(\text {parent}\left(v^{\prime}\right), v^{\prime}\right)$ 
  \STATE TrueVertex $(v^{\prime})=$ true \label{205}
  \ELSE
  \STATE $\hat g\left(v^{\prime}\right)=\hat g\left(\text {parent}\left(v^{\prime}\right)\right) + c_{x,v}\left(\text {parent}\left(v^{\prime}\right), v^{\prime}\right)$\label{201}
  \ENDIF 
  \IF{$ \nexists v^{\prime \prime} \in Q_o$ s.t. $\operatorname{Conf}\left(v^{\prime \prime}\right)=\operatorname{Conf}\left(v^{\prime}\right) \AND$
TrueEdge$\left(v^{\prime \prime}\right) \AND \hat g\left(v^{\prime \prime}\right) \leq \hat g\left(v^{\prime}\right) \AND v^{\prime} \notin Q_c$}\label{line13}
\STATE $\hat f\left(v^{\prime}\right)=\hat g\left(s^{\prime}\right)+ \hat h\left(v^{\prime}\right)$, $Q_o \stackrel{+}{\leftarrow}\{v^{\prime}\}$\label{line16}
\ENDIF\label{line20}
  \ENDFOR\label{line21}
  \RETURN $Q_o, Q_c$
 \end{algorithmic} 
 \end{algorithm}

In \alg{expansion}, the expanded state has both the true vertex and the true edge cost so that LTO puts all successors of $v$ in $Q_o$. GetSuccessors generates a copy of each neighboring state to maintain the states from different parents (line~\ref{line7}).
The same vertex that originated from other parent states might have already figured out the true configuration of the vertex
by already running TO. Hence, LTO checks if other versions of the successor state $v^{\prime}$ have the true configuration in $Q_o, Q_c$ (line~\ref{200}). If true, 
we update $\hat g(v^{\prime})$ with $c_{TO}$. Thus, $\mathcal{G}$ has an expensive cost for edges with many integer variables.
We also set TrueVertex$(v^{\prime})$ to true (line~\ref{205}). If another version of the successor state $v^{\prime}$ does not have the true configuration, we use a distance from $v$ to the voxel's edge where $v^{\prime}$ belongs as a cost of the edge to guarantee the bounded suboptimality (line~\ref{201}). LTO checks if this version of $v^{\prime}$ should be considered for maintaining in $Q_o$ (line~\ref{line13}). If there exists the state $v^{\prime \prime}$ that represents the same configuration of $v^{\prime}$ with the true edge cost and the lower $\hat g$ value, we do not maintain $v^{\prime}$.

\subsection{Edge generation (\alg{updateedge})}
In \alg{updateedge}, the state has the true vertex but does not obtain the true edge cost yet. 
Let $v^x_y$ and $v^z_y$ be the vertices representing the same configuration but originated from different parents $x, z$.
Here,  $e(v^{a}_b \leftrightarrow v^{c}_d)=e(v^{e}_{b} \leftrightarrow v^{f}_{d})$. Hence, we do not want to run expensive TO again to get $e(v^{a}_b \leftrightarrow v^{c}_d)$ if we already obtain $e(v^{e}_{b} \leftrightarrow v^{f}_{d})$. 
On line~\ref{300}, $\operatorname{CheckSamePair}$ checks if we already obtain the same configuration pair from different parents. If true, we get the same configuration pair (line~\ref{GetSamePair}). Line~\ref{edge_checking} checks if the obtained edge is feasible. If true,
we set $\operatorname{TrueEdge}(v)$ to true, get the cost (line~\ref{311}) and use it to update the $\hat g(v)$ (line~\ref{312}). 
We use $c$ instead of $c_{TO}$ because LTO already figures out the true edge cost, and it does not make sense for the edge cost to be expensive due to $\omega ^{n_i}$. On lines~\ref{line30}-\ref{line33}, like line~\ref{line13} in \alg{expansion}, we insert $v$ into $Q_o$ if no states exist satisfying the if condition. If false on line~\ref{300}, LTO runs TO (line~\ref{line2300}) and perform the same action between lines~\ref{311}-\ref{line33}.
On line~\ref{301}, GetWarmStart computes the initial guesses $w_{\operatorname{opt}}$ of each decision variable. 

\subsection{Vertex Validation (\alg{updatevertex})}
Since the state does not have the true vertex, LTO runs TO and gets the true configuration of $v$. The structure of \alg{updatevertex} and \alg{updateedge} is essentially the same, but in  \alg{updatevertex}, we use $c_{TO}$ as the cost to avoid the expensive edge generation (line~\ref{4000}). 

  \begin{algorithm}[t]
    \small 
  \algsetup{linenosize=\small}
 \caption{$\operatorname{UpdateEdge}(\mathcal{G}, Q_o, Q_c, v)$}
 \label{updateedge}
 \begin{algorithmic}[1]
 \IF{$\operatorname{CheckSamePair}(\mathcal{G}, \operatorname{parent}(v), v)$ is False}\label{300}
 \STATE $w_{\operatorname{opt}}=\operatorname{GetWarmStart}(\mathcal{G}, \operatorname{parent}(v), v)$\label{301}
  \STATE $c, \mathcal{G}, \text{Edge}=\operatorname{RunTO}(\mathcal{G},\operatorname{parent}(v), v, w_{\operatorname{opt}})$\label{line2300}
  \ELSE
  \STATE $c, \mathcal{G}, \text{Edge} = \operatorname{GetSamePair}(\mathcal{G}, \operatorname{parent}(v), v)$\label{GetSamePair}
  \ENDIF
 \IF{\text{Edge is feasible}} \label{edge_checking}
 \STATE TrueEdge $(v)=$ true, $c=\operatorname{CostSamePair}(\operatorname{parent}(v), v))$\label{311}
 \STATE $\hat g(v)=\hat g(\text {parent}(v))+c$\label{312}
 \IF{$ \nexists v^{\prime \prime} \in Q_o$ s.t. $\operatorname{Conf}\left(v^{\prime \prime}\right)=\operatorname{Conf}\left(v\right) \AND$
TrueEdge$\left(v^{\prime \prime}\right) \AND \hat g\left(v^{\prime \prime}\right) \leq \hat g\left(v\right)$}\label{line30}
\STATE $\hat f\left(v\right)=\hat g\left(v\right)+ \hat h\left(v\right)$, $Q_o \stackrel{+}{\leftarrow}\{v\}$
\label{line31}
\ENDIF\label{line33}
  \ENDIF\label{line35}
  \RETURN $\mathcal{G}, Q_o, Q_c$
 \end{algorithmic} 
 \end{algorithm}

  \begin{algorithm}[t]
    \small 
  \algsetup{linenosize=\small}
 \caption{$\operatorname{UpdateVertex}(\mathcal{G}, Q_o, Q_c, v)$}
 \label{updatevertex}
 \begin{algorithmic}[1]
 \IF{$\operatorname{CheckSameVertex}(\mathcal{G}, v)$ is False}\label{4400}
 \STATE $\mathcal{G},  \text{Configuration}=\operatorname{RunTO}(\mathcal{G}, v)$\label{line23}
 \ELSE
 \STATE $\mathcal{G},  \text{Configuration}=\operatorname{GetSameVertex}(\mathcal{G}, v)$
 \ENDIF
 \IF{\text{Configuration is feasible}}
 \STATE TrueVertex $(v)=$ true \label{check_vertex}
 \STATE $g(v)=g(\text {parent}(v))+c_{TO}\left(\text {parent}\left(v\right), v\right)$ \label{4000}
 \IF{$ \nexists v^{\prime \prime} \in Q_o$ s.t. $\operatorname{Conf}\left(v^{\prime \prime}\right)=\operatorname{Conf}\left(v\right) \AND$
TrueEdge$\left(v^{\prime \prime}\right) \AND g\left(v^{\prime \prime}\right) \leq g\left(v\right)$}\label{line40}
\STATE $\hat f\left(v\right)=\hat g\left(v\right)+ \hat h\left(v\right)$, $Q_o \stackrel{+}{\leftarrow}\{v\}$\label{line311}
\ENDIF\label{line433}
\ENDIF
  \RETURN $\mathcal{G}, Q_o, Q_c$
 \end{algorithmic} 
 \end{algorithm}

\section{Formal Analysis}

\subsection{Complexity}\label{CT_Proof}
We identify line~\ref{line2300} in \alg{updateedge} and line~\ref{line23} in \alg{updatevertex} as the main sources of planning time.
For planning with expensive edge generation, it makes more sense to discuss the time complexity based on TOs. 

\begin{thm}\thlabel{micp_complexity}
Let $\mathcal{X}$ be the configuration space normalized to $[0,1]^{d}$, where $d \in \mathbb{N}$. Let $K, r$ be the number of intervals along each axis and the normalized distance calculated as $\ell_{\infty}$ (see \fig{edge}). Then, the TO-aware time complexity is $O(\left(2i+1\right)^dK^d)$ where $i = 0,1,\cdots, K$, $r = (1/K)i$.
\end{thm}
\begin{proof}
\fig{edge} shows the case where $d=2, K=4, r=1/4, i = 1$.
The total number of vertices is $K^d$ so TO for finding a vertex is called at most $K^d$ times. 
For the edge generation, LTO connects the vertices if the $\ell_{\infty}$ between the center of the voxel to which one vertex belongs and the center of the voxel to which the other vertex belongs to is less than or equal to $r$. The total number of edges per vertex is $(2i+1)^d-1$ (the vertices inside the green rectangle in \fig{edge} except for the vertex (d)). 
Thus, the total number of TO to find the edge is  $((2i+1)^d-1)K^d$. 
%
%
\end{proof}
We can even bound $K$ when TO uses MICP. 
\begin{cor}\thlabel{corollcomplexity}
$K$ is bounded as:
$T_o(2^{B}K^d + 2^{NB(i+1)^d}((2i+1)^d-1))\leq T_t$
where $B$ is the maximum number of integer variables in a voxel, $T_o$ is the average solving time of convex programming on a problem domain with no integer variables and $N=0$, and $T_t$ is the acceptable running time.
\end{cor}
\begin{proof}
When finding a true vertex, the MICP solver using B\&B searches at most $2^B$ solutions and solves the regular convex programming for each solution. In the worst-case, it solves convex programs for all voxels and the solving time is $2^BK^dT_o$. Regarding the edge generation, the solver investigates  $(i+1)^d$ voxels at most for every single edge. When finding edges, we consider $N$ steps planning problem so that the total number of integer variables when finding an edge is $NB(i+1)^d$. Because the solver runs at most $(2i+1)^d-1$ per a voxel to find an edge, the worst-case solving time for edge generation is $2^{NB(i+1)^d}((2i+1)^d-1)$. 
\end{proof}








\subsection{Optimality}

We can bound the cost of the solution as follows:


\begin{thm}\thlabel{exact}
Let  $\xi_{}^{*}$ be an optimal path. 
LTO return a path $\xi$ with cost $c(\xi)\leq \alpha c(\xi_{}^{*})$ with $\alpha =  (1+\omega^{M})$ where $M=NB(i+1)^d$.
\end{thm}
\begin{proof}
%
We need to show $\hat{g}(v) \leq \alpha g^{*}(v)$. 
We use induction. At the start of planning,  
$\hat{g}(v_{\operatorname{start}}) = g^{*}(v_{\operatorname{start}}) \leq \alpha g^{*}(v_{\operatorname{start}})$
so the base case holds. Next, after some iteration of \alg{alg2} and assume that  $\hat{g}(v) \leq \alpha g^{*}(v)$ holds for all $v \in \xi$ so far. 
Let $v_p \in \xi$ with $\hat{g}(v_p) > \alpha g^{*}(v_p)$, resulting in $\hat{f}(v_{p})>\alpha(g^{*}(v_{p}))+\hat h(v_p)$. $\hat{g}(v) \leq \alpha g^{*}(v)$ holds if no such $v_p$ exists. Here, we show that LTO will not choose such $v_p$ on line~\ref{extract} in \alg{alg2} even if $v_p$ exists. 

Case 1: A vertex has been expanded before $v_p$ along $\xi$. In this case, We must have a  
 $v_{a-1} \in \xi$ before $v_p$ along $\xi$ with successor $v_a$ on $Q_o$. If TrueEdge($v_a$) is true:

$\begin{aligned} \hat{g}(v_{a}) & \leq \hat{g}(v_{a-1})+c(v_{a-1}, v_{a}) \\ & \leq \alpha g^{*}(v_{a-1})+c(v_{a-1}, v_{a}) \leq \alpha g^{*}(v_{a}) \end{aligned}$\\
If TrueEdge($v_a$) is false:

$\begin{aligned} \hat{g}(v_{a}) & \leq \hat{g}(v_{a-1})+c_{TO}(v_{a-1}, v_{a}) \\
& \leq \omega^{n_i}  (g^{*}(v_{a-1})+c(v_{a-1}, v_{a})) \\
& \leq \omega^M  (g^{*}(v_{a-1})+c(v_{a-1}, v_{a}))  = \alpha g^{*}(v_{a}) \end{aligned}$\\
Hence, the assumption $\hat{g}(v) \leq \alpha g^{*}(v)$ holds true for all iterations. Since for every vertex $\alpha g^{*}\left(v_{i}\right)+\hat h\left(v_{i}\right)\leq \alpha g^{*}\left(v_{i+1}\right)+\hat h\left(v_{i+1}\right)$ is true due to the consistency of $\hat h$,\\
$\begin{aligned} \hat{f}\left(v_{a}\right) &=\hat{g}\left(v_{a}\right)+ \hat h\left(v_{a}\right) 
 \leq \alpha  (g^{*}(v_{a}))+\hat h(v_{a}) <\hat{f}\left(v_{p}\right) 
 \end{aligned}$\\
 which means that $v_p$ will not be chosen. 
 
 Case 2: No expanded vertex before $v_p$ along $\xi$. In this case, $Q_o$ must contain the start vertex, where we can apply the same discussion above, resulting in $\hat f(v_{\operatorname{start}})<\hat f(v_p)$.
 
 Finally, 
 $c(\xi) = \hat g(v_{\operatorname{goal}}) \leq \alpha g^{*}(v_{\operatorname{goal}}) \leq \alpha c(\xi_{}^{*})$.
\end{proof}

\section{Numerical Experiments}
We validate LTO on two motion planning problems: free-flying robots in $\mathbb{R}^{2}$ and legged robots in $\mathbb{R}^{21}$. 
We test algorithms with ten trials.

We evaluate LTO of the different numbers of voxels with/without the warm-start. We set $r$ such that a vertex has 15 edges per vertex for the free-flying robot experiments and 20 edges for the legged robot experiments. 
With warm-start options, LTO uses other already investigated trajectories by TO if $d_{\operatorname{cost}}$ in configuration space is less than 0.1. 
We also run the regular TO (i.e., no GSP is embedded), other GSP (i.e., weighted A*, LWA*), and SBP (i.e., PRM \cite{PRM}, lazyPRM \cite{LazyPRM}, RRT \cite{RRT}, RG-RRT \cite{RG-RRT}). 
To have a fair comparison, 
we incorporate TO into GSPs and SBPs. 
When they find a node and connect nodes, they use TO to check if the sampled configuration and the edge is feasible. 

We use Gurobi \cite{gurobi} to solve MICP on Intel Core i7-8750H machine and implement all planning codes in Python. 

\begin{figure}
    \centering
    \includegraphics[width=0.4865795\textwidth, clip]{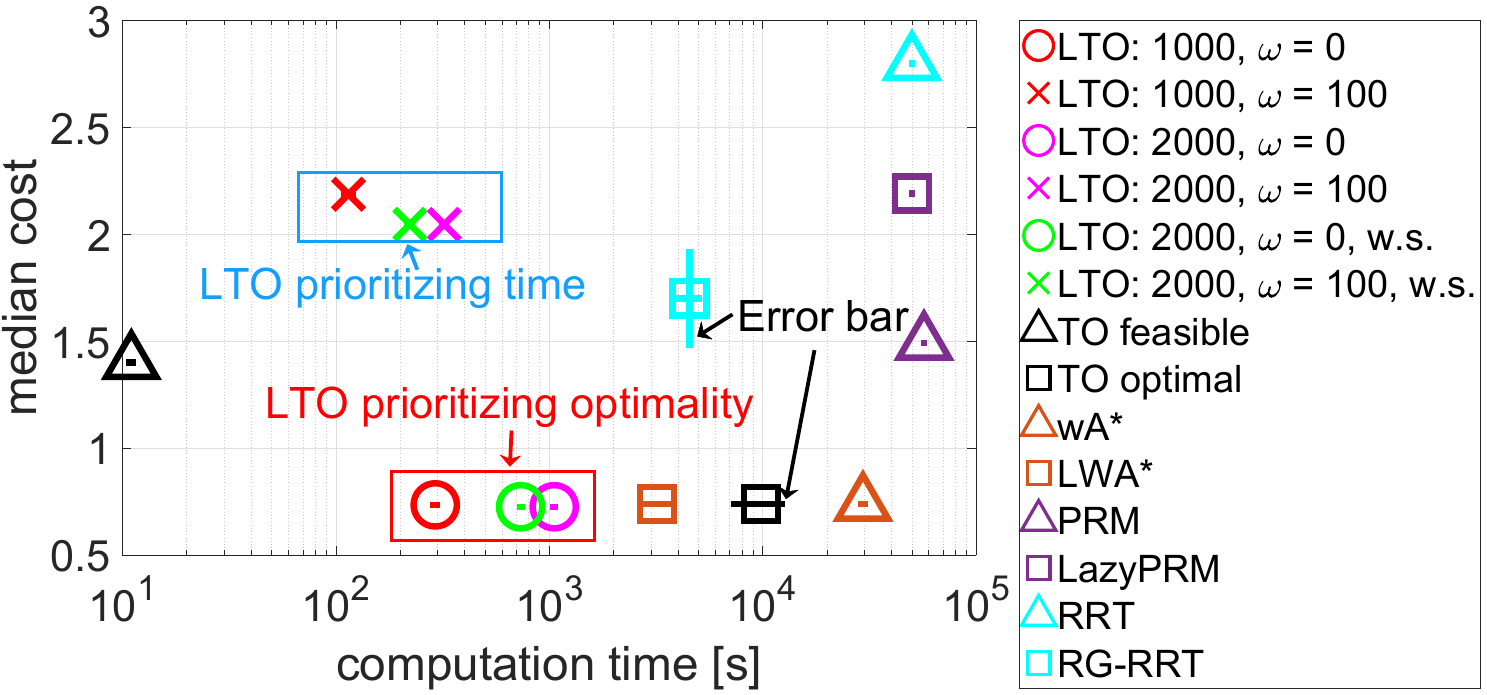}
    \caption{The results of Section~\ref{free-fly-section}. Error bars represent a 95 \% confidence interval for a Gaussian distribution. Note that for some algorithms, the confidence intervals are very small and are not visible. Compared with TO, LTO prioritizing optimality (i.e., $\omega=0$) finds the optimal solution about nine times faster without sacrificing the solution cost much (1.2 $\%$ worse).}
    \label{2d_fraction}
\end{figure}

\subsection{Free-Flying Robots}\label{free-fly-section}

We consider a free-flying robot in $\mathbb{R}^{2}$ with multi obstacles. We define $p_{t} \in \mathbb{R}^{2}$ as the position and $v_{t} \in \mathbb{R}^{2}$ as the velocity. The state $x_t = (p_t,v_t)$ is controlled by $u_t\in \mathbb{R}^{2}$. Thus, the robot solves the following MICP from $x_{s}$ to $x_{g}$ while remaining in the safe region $\mathcal{X}_{\text {safe}}$ \cite{free-fly}:
\begin{equation}
\begin{array}{cl}
\text {minimize} & \sum_{t=0}^{N-1}\left(x_{t}-x_{g}\right)^{\top}Q\left(x_{t}-x_{g}\right)+u_{t}^{\top}u_{t}\\
\text {s.t.} & x_{t+1}=A x_{t}+B u_{t},  t=0, \ldots, N-1\\
& \left\|u_{t}\right\|_{2} \leq u_{\max }, \quad t=0, \ldots, N-1\\
& x_{\min } \leq x_{t} \leq x_{\max }, \quad t=0, \ldots, N\\
& x_{0}=x_{s}, \quad  x_{N}=x_{g}\\
& x_{t} \in \mathcal{X}_{\text {safe}}, \quad t=0, \ldots, N\label{space_craft}
\end{array}
\end{equation}
The constraints $x_{t} \in \mathcal{X}_{\text {safe}}$ are represented using a big-M formulation with binary variables $z$ \cite{free-fly}.
We consider axis-aligned rectangular obstacles. 
We run LTO under 1000 and 2000 voxels with  $\omega=0, 100$ and with/without a warm-start. We set $N=7$ for TO inside LTO. For TO, we use $N=70$, which is the minimum number of $N$ for us to find a solution. The number of continuous variables, binary variables, and constraints is 168, 3360, 4230, respectively. 
We run SBP and GSP at most 5000 samples with different grid sizes and show the results for the grid size showing the optimal cost. 



The solution cost versus planning time is plotted in \fig{2d_fraction}. LTO finds as good solutions as TO finds with decreased planning time. 
The generated trajectories are shown in \fig{concept}. By navigating the robot to a region with fewer integer variables (i.e., fewer obstacles), LTO can generate robot trajectories quickly. 
We also observe that the lower the inflation factor $\omega$ is, the more optimal trajectory LTO generates.
It takes more time to design the trajectory since LTO does not guide the robot to avoid the computationally expensive regions, but it still quickly generates the trajectory without sacrificing the solution cost so much compared with the long-horizon TO.

\begin{figure}
    \centering
    \includegraphics[width=0.325\textwidth, clip]{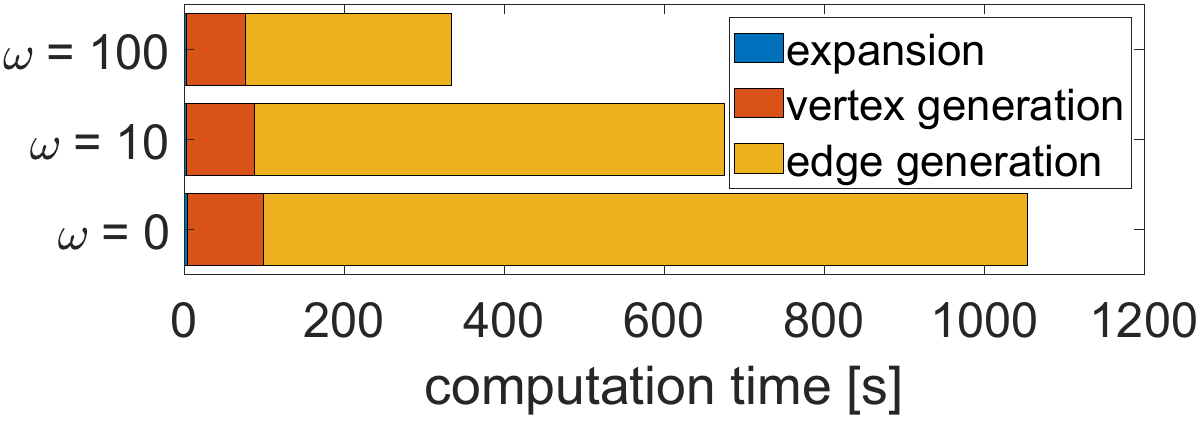}
    \caption{Consumed time with 2000 voxels in $\mathbb{R}^{2}$ for LTO. The larger the inflation factor $\omega$ is,
    the less time LTO spends by avoiding expensive edge generation.}
    \label{time_comparison}
\end{figure}

\begin{figure*}[!t]
    \centering
    \includegraphics[width=0.813\textwidth, clip]{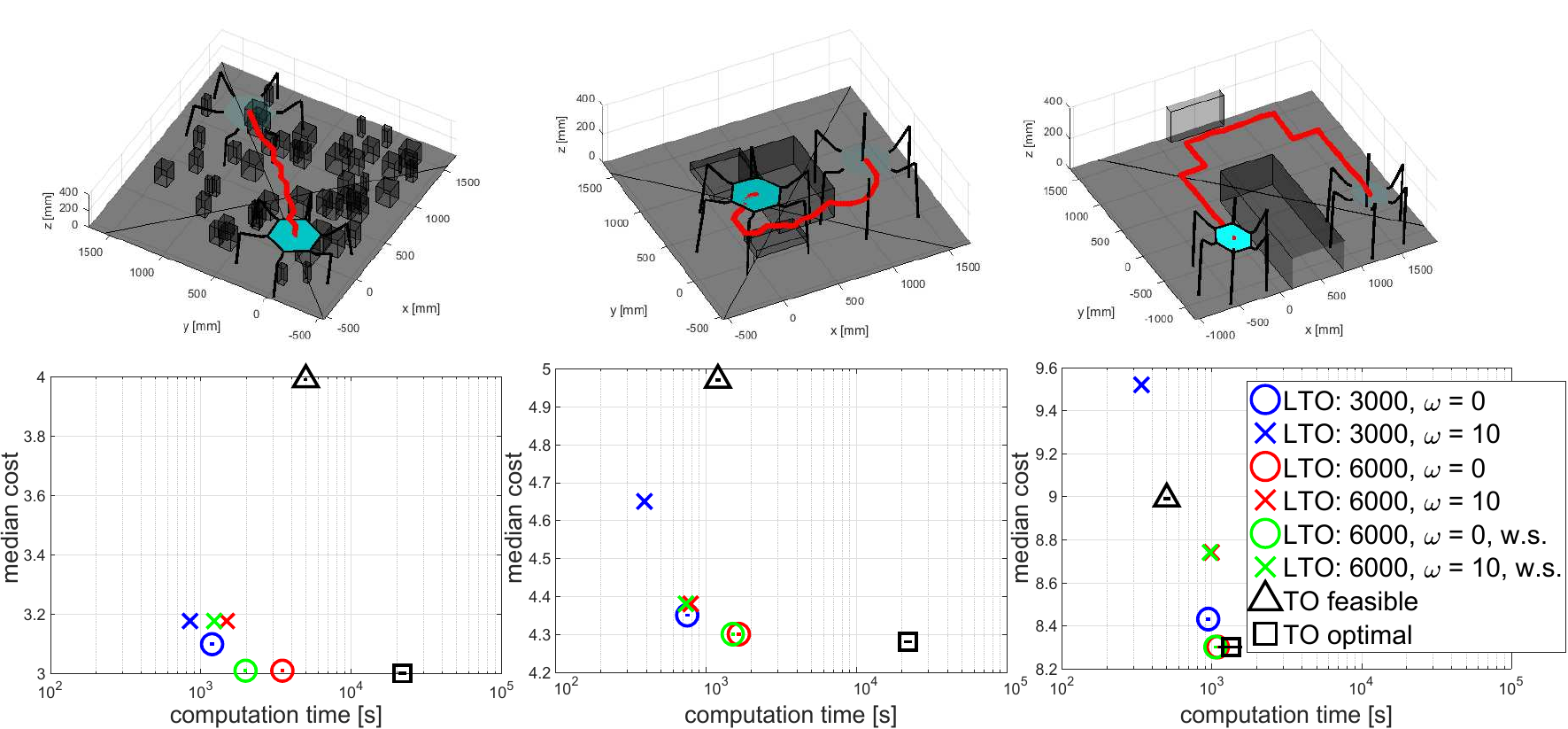}
    \caption{The results of Section~\ref{legged_robot_result_section} in $\mathbb{R}^{21}$. The red lines in the top figure indicate the body trajectory. Error bars represent a 95 \% confidence interval for a Gaussian distribution. Note that for LTO, the confidence intervals are very small and are not visible. LTO with 6000 voxels with the warm-start finds the optimal solution about 11.3, 14.0, 1.4 times faster than TO without degrading the solution cost so much (about 0.33, 0.35, 0.01 $\%$ worse), from the left to the right environment, respectively. By increasing the inflation factor $\omega$, LTO can generate globally suboptimal trajectories faster than TO feasible option.}
    \label{legged_robot_fig}
\end{figure*}

Here, we compare the results of LTO with TOs. 
While the best planner among LTOs (the warm-start, 2000 voxels, $\omega=0$) designs the trajectory that cost is 1.2 \% worse than TO with a cost function, it decreases the computation time by about 93 \%. 
TO with $N=70$ is the simplest TO we get the feasible global trajectories. We manually increase $N$ until we get the feasible trajectory. 
In contrast, since LTO iteratively solves the short-horizon TO, 
we do not need to spend time tuning $N$ until we get global feasible trajectories, resulting in less offline user time consumption.
Also, the variance of the planning time is smaller than that of TO.
Recognizing that the solver's behavior in terms of solving time is more uncertain for the large scale optimization problem \cite{Variability}, LTO uses the small scale problem (i.e., short-horizon TO) to have the small variance of the planning time. 


We also evaluate our algorithms in terms of parameters. \fig{2d_fraction} shows 
that $\omega$ provides a tuning knob to trade off between the solution cost and the planning time. 
\fig{time_comparison} shows the individual time cost on \alg{expansion}-\alg{updatevertex} within LTO with different $\omega$ for 2000 voxels. It shows that TO for generating an edge spends a large amount of time. It also indicates that by increasing $\omega$, LTO spends less time to generate edge by navigating the robot to the region with fewer discrete variables, resulting in less total planning time.
As for the number of voxels, 
 LTO with 1000 voxels and $\omega=0$ designs the trajectory, resulting in 72 \% decreased planning time with the 0.4 \% worse solution cost compared with LTO with 2000 voxels and $\omega=0$.
As a result, we may use the solution from LTO with 1000 voxels and $\omega=0$. 


All algorithms in SBP and GSP but RG-RRT spends a large amount of time. While LTO and TO have the success rate of 100 $\%$, SBP and GSP have the success rate of 20 $\%$ except for RG-RRT, which has that of 40 $\%$.
This is because the regular SBPs and GSPs do not consider dynamics.

\subsection{Legged Robots}\label{legged_robot_result_section}

We consider a $M$-legged robot motion planning problem. We denote the body position as $q_{t} \in \mathbb{R}^{3}$, its orientation as $\theta_{t} \in \mathbb{R}^{3}$, and toe $i$ position as $p_{it} \in \mathbb{R}^{3}$. To realize a stable locomotion, we consider the reaction force $f_{it}^{r} \in \mathbb{R}^{3}$  at foot $i$. Thus, the robot solves the following MICP \cite{MIT_envelope}: 
\begin{equation}
\begin{array}{cl}
\text { minimize } & \sum_{t=0}^{N-1}\left(x_{t}-x_{g}\right)^{\top}Q\left(x_{t}-x_{g}\right)\\
\text {s.t.} & x_{\min } \leq x_{t} \leq x_{\max }, \quad t=0, \ldots, N\\
&  |\Delta x_{t} | \leq \Delta x, \quad t=1, \ldots, N\\
& {p}_{it} \in \mathcal{R}_{i}({q_t}, {\theta_t}), \quad t=0, \ldots, N\\
& x_{0}=x_{s}, \quad  x_{N}=x_{g}\\
& x_{t} \in \mathcal{X}_{\text {safe}}, \quad t=0, \ldots, N\\
&\sum_{i=1}^{M} f_{it}^{r}+{F}_{t o t}={0}\\
& \sum_{i=1}^{M}\left({p}_{it} \times {f}_{it}^{r}\right)+{M}_{t o t}={0}
\label{legged_robot}
\end{array}
\end{equation}
where $x_t$ contains kinematics-related decision variables $q_t, \theta_t, p_{1t}, \cdots, p_{Mt}$. Here, ${p}_{it} \in \mathcal{R}_{i}({q_t}, {\theta_t})$ shows kinematics constraints. $\sum_{i=1}^{M} f_{it}^{r}+{F}_{t o t}={0}$ and $\sum_{i=1}^{M}\left({p}_{it} \times {f}_{it}^{r}\right)+{M}_{t o t}={0}$ represent the static equilibrium of force and moment constraints, respectively. For kinematics constraints, we approximate them as linear constraints \cite{ETHNLP}. 
The bilinear terms of the static equilibrium of moment can be represented as piecewise McCormick envelopes with binary variables $z$ \cite{MIT_envelope}. 
%
In this work, we just consider $q_t, p_{it}$ for planning. To consider orientation, we may utilize \cite{tyler}-\cite{multi}.  We consider a six-legged robot with 3 DOF per leg, resulting in 21 DOF planning.
We run LTO under 3000 and 6000 voxels with $\omega=0, 10$ and with/without a warm-start. We set $N=7$ for TO inside LTO. For TO, we use $N=56, 63, 70$ from the left to the right environment in \fig{legged_robot_fig}, respectively, which are the minimum number of $N$ for us to find a solution. The number of continuous variables is 20220, 23106, 25674, the number of binary variables is 32400, 14190, 6754, and the number of constraints is 74680, 55827, 48532, from the left to the right environments, respectively.


The generated trajectories and the solution cost versus planning time are shown in \fig{legged_robot_fig}. In the left and middle environments, LTO shows better performance compared with TO. In \fig{2d_fraction}, TO feasible planning finds the trajectory most quickly, but \fig{legged_robot_fig} shows that it spends more time to just find a feasible solution if the planning problem is more difficult. In the right environment, TO optimal and LTO show similar results. Since the right environment has fewer discrete variables, we do not observe the advantage of using LTO. Therefore, LTO works well in environments with many discrete decision variables like the left and the middle environments.

\section{Conclusion}

We presented LTO for high DOF robots in cluttered environments. Because LTO deeply unifies TO and GSP algorithms, it can consider the original long-horizon TO problem with a decreased planning time. We proposed a TO-aware cost function that considers the difficulty of TO.
Furthermore, LTO employs other edges in the graph as a warm-start to accelerate the planning process.
We also presented proofs of the complexity and optimality. Finally, we performed planning experiments of a free-flying robot and a legged robot motion planning problems, showing that LTO is faster with a small variance of the planning time.

Since LTO has a small variance in planning time, we argue that it would be useful for safety-critical applications such as autonomous driving. Additionally, it consists of TO and GSP so that users can use other planning algorithms for each subcomponent according to their specifications.

\end{document}